\documentclass{article}
\ifdefined\pdfminorversion\pdfminorversion=7\fi
\usepackage{iclr2027_conference}
\usepackage{times}
\usepackage{amsmath,amssymb,amsthm,mathtools}
\usepackage{booktabs,multirow,array}
\usepackage{graphicx}
\usepackage{float}
\usepackage{placeins}
\usepackage{needspace}
\usepackage[font=small,labelfont=bf]{caption}
\renewcommand{\arraystretch}{1.12}
\usepackage[table]{xcolor}
\usepackage{microtype}
\usepackage{xspace}
\usepackage{hyperref}
\usepackage{url}

\newcommand{\method}{\textsc{EESD}\xspace}
\newcommand{\eed}{\textsc{EED}\xspace}
\newcommand{\fix}{\textsc{Fix}\xspace}
\newcommand{\regress}{\textsc{Regress}\xspace}
\newcommand{\preserve}{\textsc{Preserve}\xspace}
\newcommand{\unresolved}{\textsc{Unresolved}\xspace}

\newcommand{\E}{\mathbb{E}}

\newcommand{\ind}{\mathbb{I}}
\newtheorem{proposition}{Proposition}

\definecolor{cardblue}{RGB}{235,244,252}
\definecolor{cardgreen}{RGB}{237,248,241}
\definecolor{cardpurple}{RGB}{245,240,250}
\definecolor{cardorange}{RGB}{252,245,235}
\definecolor{posgreen}{RGB}{226,242,231}
\definecolor{negred}{RGB}{249,231,231}
\definecolor{neutralgray}{RGB}{243,244,246}

\hypersetup{colorlinks=true,linkcolor=blue,citecolor=blue,urlcolor=blue,
  pdftitle={How Much Evidence Should a Coding Agent's Self-Correction Carry? Adaptive Dirichlet Evidence for Self-Distillation},
  pdfauthor={Yunbo Long, Guangya Hao, Yuhan Liu, Yiting Duan,
Longyan Tan, Yunchen Long, Hao Wu},
  pdfsubject={Preprint},
  pdfkeywords={coding agents, language-model agents, self-correction,
  recursive self-improvement, execution feedback, self-training,
  self-distillation, code generation, program repair, Dirichlet evidence}}

\title{How Much Evidence Should a Coding Agent's Self-Correction Carry?
Adaptive Dirichlet Evidence for Self-Distillation}

\author{
{\bfseries
Yunbo Long$^{1}$, Guangya Hao$^{1}$, Yuhan Liu$^{1}$, Yiting Duan$^{2}$}\\
{\bfseries
Longyan Tan$^{1}$, Yunchen Long$^{3}$, Hao Wu$^{2}$}\\[0.5em]
{\normalfont\small
$^{1}$University of Cambridge \quad
$^{2}$Western Sydney University}\\
{\normalfont\small
$^{3}$Guangdong University of Technology}
}
\date{}
\iclrfinalcopy

\begin{document}
\maketitle
\lhead{Preprint}
\raggedbottom

\begin{abstract}

Execution feedback lets coding agents revise programs and learn from their own corrections. A correction's learning weight should reflect both the transitions supported by its executions and the amount of evidence behind that support. We introduce \emph{Effective-Evidence Self-Distillation} (\method), which represents these quantities separately. Normalized execution relevance determines relative transition support and an effective pseudo-count mass; a Dirichlet posterior then produces an uncertainty-penalized weight for KL-anchored correction learning. Under a symmetric prior, changing mass preserves category ordering, and effective mass yields a supervised coefficient bounded by its matched fixed-mass counterpart. Across four model--domain history sweeps, increasing visible observations from one to eight reduces future-outcome NLL by $55.0$--$59.3\%$. At eight observations, effective mass achieves lower NLL than fixed mass in all four comparisons. In the primary matched DeepSeek/RunBugRun study, argmax predictions agree on all 3,000 examples, with the largest NLL gain under concentrated relevance. After one correction-learning round, DeepSeek/CodeARC all-tests Pass@1 increases from $15.0\%$ to $20.4\%$, with a paired 95\% source-bootstrap interval of $[+2.8,+8.0]$ percentage points. The twelve-setting downstream evaluation establishes the model--domain scope of this update. These results show how separating evidence support from evidence mass changes probability estimation and correction learning in coding agents.

\end{abstract}

\section{Introduction}

Coding agents use language models to generate and revise programs from execution feedback, including compiler errors and unit-test outcomes. This supports self-correction in code generation and program repair. Self-Refine and Reflexion reuse model-generated feedback or reflection at inference time \citep{madaan2023selfrefine,shinn2023reflexion}; Self-Debugging revises programs after observing execution behavior \citep{chen2024selfdebug}; and execution-guided repair, search, and verification systems use tests to select or refine candidate programs \citep{ni2023lever,tang2024rex,li2025codeprm,li2025sstar,tao2026execcritic}. These methods use feedback to improve the next candidate program.

Self-training also requires a decision about how strongly to learn from each trajectory. Correctness filtering, reward weighting, and confidence-aware training assess the quality of generated behavior \citep{jiang2025importance,jang2025corepo,yang2025dnpo}. Execution evidence supplies two further quantities: \emph{which before/after transitions the observations support} and \emph{how much evidence supports those transitions}. Representing both quantities makes the correction weight sensitive to evidence concentration even when relative support stays fixed.

\begin{figure}[t]
\centering
\includegraphics[width=\linewidth]{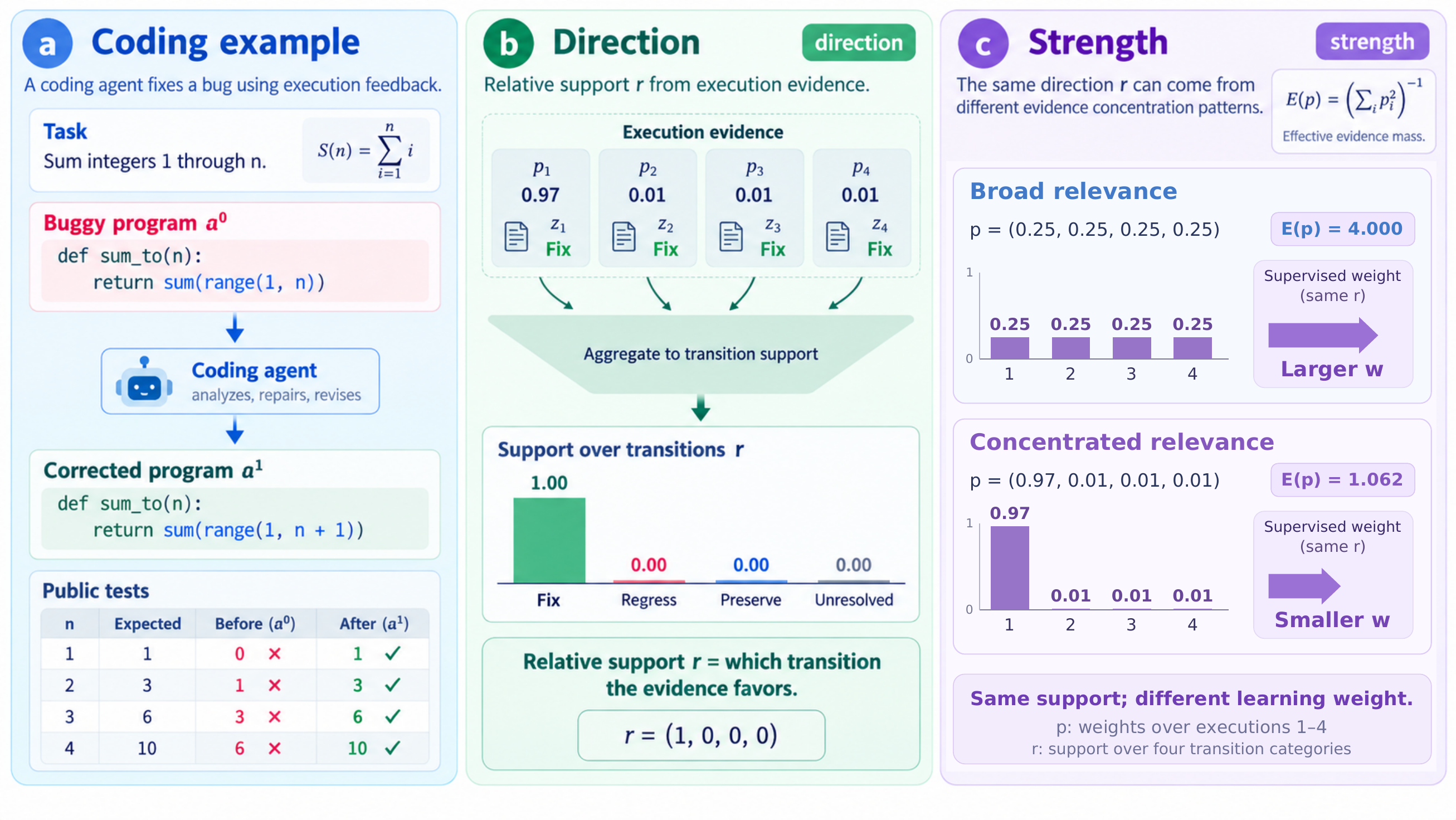}
\caption{\textbf{Direction and strength in a coding agent's self-correction.}
(a) Fixing an off-by-one range error changes four public executions from fail to pass.
(b) Aggregating execution relevance by transition gives $r=(1,0,0,0)$: all support favors \fix.
(c) The same transitions under uniform and concentrated relevance give $E(p)=4.000$ and $1.062$; at fixed prior and utility settings, they yield different supervised weights.
The $p$ bars index executions, whereas the $r$ bars index transition categories. The relevance vectors define an analytical example at fixed transition support.}
\label{fig:problem}
\end{figure}

Figure~\ref{fig:problem} illustrates the distinction through an off-by-one repair. Extending a Python range endpoint makes all four public executions pass, so relative support is $r=(1,0,0,0)$ under either relevance pattern. Uniform relevance assigns comparable weight to every execution and gives $E(p)=4$; concentrated relevance $p=(0.97,0.01,0.01,0.01)$ gives $E(p)\approx1.062$. A fixed-mass control restores total mass to four in both cases. \method instead lets relevance concentration change the supervised weight while retaining the same correction target. The analytical relevance vectors isolate the effect of concentration at fixed transition support.

The problem becomes more consequential when corrections are written back into model parameters. A mistaken inference-time correction affects one answer; an over-weighted training correction changes the policy that will generate future answers and future training data, a feedback pattern relevant to recursive self-improvement (RSI) and self-evolving agents \citep{yin2025godel,ranaldi2026evolving,xiao2026socraticswe,li2026codeskill}. We study the correction-to-update step within this feedback process: \emph{given a self-generated correction, how much should the next policy learn from its supporting execution evidence?}

We introduce \emph{Effective-Evidence Self-Distillation} (\method), a posterior-based rule for weighting correction learning. Given nonnegative relevance scores $s_i$, \method normalizes them into $p$. Relative support $r$ describes \emph{what} transition categories the observations favor, while
$E(p)=(\sum_i p_i^2)^{-1}$ describes \emph{how much} pseudo-count mass that support contributes. A Dirichlet posterior produces an uncertainty-penalized correction utility; its positive part scales correction cross-entropy while a KL term anchors the student to the entering policy. The upstream correction generator and relevance estimator remain modular: the contribution is how their execution record is converted into a learning coefficient.
Our contributions are:
\begin{enumerate}
\item We formulate correction-level \emph{direction--strength separation}: distinguishing relative execution support from the total evidence mass assigned to that support.
\item We introduce an effective-evidence Dirichlet estimator (\eed), which converts relevance concentration into adaptive Dirichlet pseudo-count mass, and connect its posterior utility to KL-anchored LoRA self-distillation. We characterize count recovery, ordering invariance, and the resulting control of the supervised coefficient.
\item We evaluate the construction through visible-history scaling, a matched-mass mechanism study with fixed argmax, and one-round correction-based self-training followed by fresh code generation across multiple models and domains.
\end{enumerate}

\FloatBarrier
\section{Related Work}
\label{sec:related}

Table~\ref{tab:positioning} positions \method relative to representative
self-correction and self-training approaches. The central distinction
is between improving an answer, assessing its suitability for learning,
and determining how much execution evidence supports the resulting
update.

\begin{table}[H]
\centering
\caption{\textbf{Positioning by the role of feedback.} $\checkmark$: explicit use; --: absent in the cited formulation; $\circ$: varies within a grouped family. Generator updates refer to the code-generating policy. Adaptive training includes filtering, preferences, and reward weighting. The final three columns describe correction-level support--mass construction. Representative methods are discussed in Section~\ref{sec:related}.}
\footnotesize
\setlength{\tabcolsep}{3pt}
\renewcommand{\arraystretch}{1.15}
\begin{tabular*}{\linewidth}{@{\extracolsep{\fill}}>{\raggedright\arraybackslash}p{3.0cm}ccccccc@{}}
\toprule
Method & \shortstack{Inference\\revision /\\selection} & \shortstack{Generator\\update} & \shortstack{Execution\\feedback} & \shortstack{Adaptive\\training} & \shortstack{Transition\\support\\$r$} & \shortstack{Effective\\mass\\$E(p)$} & \shortstack{Support /\\mass\\separation} \\
\midrule
Self-Refine & $\checkmark$ & -- & -- & -- & -- & -- & -- \\
Reflexion (coding) & $\checkmark$ & -- & $\checkmark$ & -- & -- & -- & -- \\
Self-Debugging & $\checkmark$ & -- & $\checkmark$ & -- & -- & -- & -- \\
Execution-guided verification / reranking & $\checkmark$ & -- & $\checkmark$ & -- & -- & -- & -- \\
Importance-weighted self-training & -- & $\checkmark$ & -- & $\checkmark$ & -- & -- & -- \\
Confidence-aware self-training & -- & $\checkmark$ & -- & $\checkmark$ & -- & -- & -- \\
Reward / process-weighted learning & $\circ$ & $\circ$ & $\checkmark$ & $\circ$ & -- & -- & -- \\
\midrule
\method (ours) & -- & $\checkmark$ & $\checkmark$ & $\checkmark$ & $\checkmark$ & $\checkmark$ & $\checkmark$ \\
\bottomrule
\end{tabular*}

\label{tab:positioning}
\end{table}

\paragraph{Self-correction and execution-guided repair.}

Self-Refine iteratively revises outputs using model-generated critiques, and Reflexion reuses verbal feedback to guide later attempts \citep{madaan2023selfrefine,shinn2023reflexion}. Reflexion's coding setting and Self-Debugging use execution feedback to revise programs \citep{shinn2023reflexion,chen2024selfdebug}. Execution-guided verification and reranking evaluate candidate programs with test outcomes \citep{ni2023lever,li2025sstar}. REx uses Bayesian beliefs to allocate repair-search effort, while ExecCritic couples testing with repair-agent learning \citep{tang2024rex,tao2026execcritic}. LEVER trains an auxiliary verifier for candidate selection \citep{ni2023lever}. \method operates on the correction record produced by such upstream procedures and determines its contribution to a generator update.

\paragraph{Learning from self-generated corrections.}

Correctness filtering, importance weighting, and confidence-aware self-training select or weight generated trajectories using outcome and model signals \citep{jiang2025importance,jang2025corepo}. CodePRM learns a process evaluator from execution feedback, and iterative preference learning updates models from self-generated data \citep{li2025codeprm,yang2025dnpo}. Self-evolving systems also reuse experience across updates or skill construction \citep{yin2025godel,ranaldi2026evolving,xiao2026socraticswe,li2026codeskill}. \method adds a correction-level evidence representation: relative transition support and the pseudo-count mass assigned to that support jointly determine the training weight.

\paragraph{Evidential uncertainty and effective sample size.}

Dirichlet models separate categorical probabilities from distributional concentration and support evidential uncertainty modeling \citep{sensoy2018evidential}. The inverse squared concentration of normalized weights is an established effective-sample-size functional \citep{martino2017ess}. \method combines these tools at the level of an execution record. Relative transition support $r$ allocates pseudo-counts across categories, effective mass $E(p)$ sets their total, and posterior utility sets the supervised coefficient in KL-anchored self-distillation. This execution-to-learning construction links evidence concentration directly to the strength of a correction update.

\FloatBarrier
\section{Effective-Evidence Self-Distillation}
\label{sec:method}

\method converts a coding agent's execution feedback into a correction-specific training weight. The key is to separate \emph{relative transition support} from \emph{total evidence mass}, then use both to determine how strongly to learn from a correction.

\begin{figure}[H]
\centering
\includegraphics[width=\linewidth]{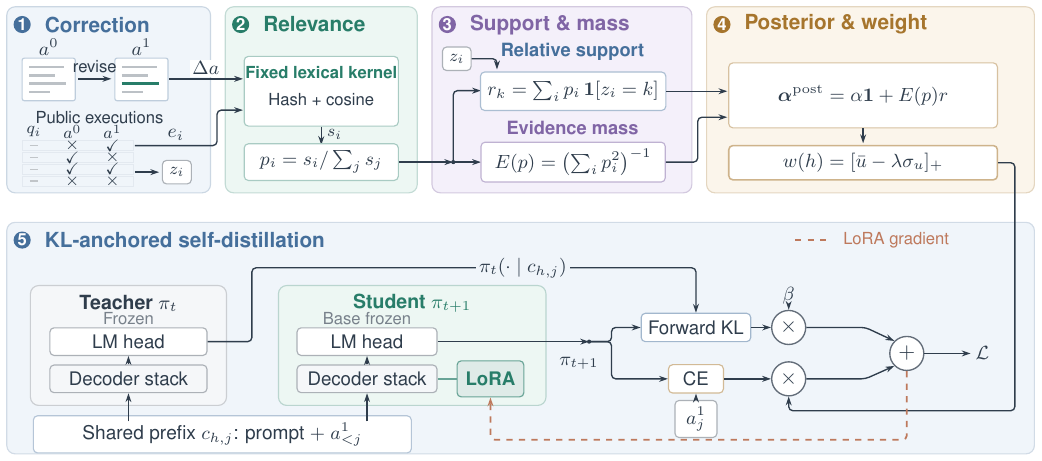}
\caption{\textbf{Effective-Evidence Self-Distillation.} The public correction record yields $p$ and transition labels $z_i$. Parallel support and mass summaries form a pseudo-count posterior and weight $w$. Teacher and student consume the same prompt and past correction tokens; the current target $a^1_j$ enters CE only. Teacher probabilities supply the forward-KL reference. The weight $w$ scales CE. The dashed coral arrow updates LoRA parameters only. Code and decoder blocks schematically represent the correction record and frozen model bases.}
\label{fig:method}
\end{figure}

\subsection{Correction evidence}
\label{sec:problem}

For a programming task $x$, a coding agent with policy $\pi_t$ produces an initial program $a^0$, and an upstream repair procedure returns a correction $a^1$. Both programs are evaluated on the same $n\ge1$ public executions $\{q_i\}_{i=1}^{n}$, yielding before/after outcomes $(o_i^0,o_i^1)$. Let $h$ denote this correction record. Each outcome pair defines a transition
\begin{equation}
z_i\in\{\fix,\regress,\preserve,\unresolved\},
\end{equation}
corresponding to fail$\rightarrow$pass, pass$\rightarrow$fail, pass$\rightarrow$pass, and fail$\rightarrow$fail, respectively.

Each execution receives a relevance score $s_i\ge0$, with $\sum_i s_i>0$. Define
\begin{align}
p_i &= \frac{s_i}{\sum_j s_j},
\label{eq:normalized-relevance}\\
r_k &= \sum_i p_i\ind[z_i=k].
\label{eq:relative-support}
\end{align}
The four-vector $r$ describes which transitions the observations support. Direction denotes the ordering of transition categories. Our implementation obtains $s_i$ from a frozen lexical kernel comparing the code change with public execution descriptors; Appendix~\ref{app:architecture} gives the details.

\Needspace{11\baselineskip}
\subsection{From effective evidence to correction utility}

Our fixed-mass control uses $c_k^{\mathrm{fixed}}=nr_k$, preserving relative support while assigning every record of length $n$ the same total mass. \eed replaces this nominal mass with
\begin{align}
E(p) &= \left(\sum_i p_i^2\right)^{-1},
\label{eq:effective-evidence}\\
\alpha_k^{\mathrm{post}} &= \alpha+E(p)r_k,
\qquad \alpha>0,
\label{eq:posterior}
\end{align}
and defines $\theta\mid h\sim\mathrm{Dir}(\boldsymbol\alpha^{\mathrm{post}})$. Thus $r$ allocates pseudo-counts across transitions, while $E(p)$ sets their total mass. We have $1\le E(p)\le n$; uniform relevance gives $E(p)p_i=1$ and recovers ordinary counts. The posterior is conditional on the relevance weights, and its pseudo-count mass summarizes their concentration.

To reward fixes and penalize regressions, set $v=(1,-1,0,0)$. Writing $\alpha_0=\sum_k\alpha_k^{\mathrm{post}}$ and $\mu_k=\alpha_k^{\mathrm{post}}/\alpha_0$, the posterior utility moments are
\begin{align}
\bar u &= \E[v^\top\theta]
=\mu_{\fix}-\mu_{\regress},
\label{eq:utility-mean}\\
\sigma_u^2 &=
\frac{\mu_{\fix}+\mu_{\regress}-\bar u^2}{\alpha_0+1}.
\label{eq:utility-variance}
\end{align}
The correction weight is
\begin{equation}
w(h)=\left[\bar u-\lambda\sigma_u\right]_+,
\qquad \lambda\ge0,
\label{eq:utility}
\end{equation}
where $[a]_+=\max(a,0)$. The uncertainty penalty reduces the supervised contribution of corrections with less decisive posterior utility; nonpositive utility yields zero supervised weight. The score combines expected utility with a posterior standard-deviation penalty.

For the all-\fix example in Figure~\ref{fig:problem}, $\alpha=0.1$ and $\lambda=0.5$ give weights $0.832$ and $0.542$ for uniform and concentrated relevance. The matched fixed-mass control assigns $0.832$ to both. Appendix~\ref{app:coding-illustration} gives the calculation. The example holds correction targets and KL coefficients fixed.

\subsection{KL-anchored self-distillation}

The student learns from $a^1$ while a frozen copy of $\pi_t$ supplies a forward-KL anchor:
\begin{equation}
\mathcal L_{\method}(h)
=
w(h)\mathcal L_{\mathrm{CE}}(h)
+\beta D_{\mathrm{KL}}^{(h)}(\pi_t\Vert\pi_{t+1}).
\label{eq:eesd-loss}
\end{equation}
Here $\mathcal L_{\mathrm{CE}}$ is correction cross-entropy, and $D_{\mathrm{KL}}^{(h)}$ is forward KL averaged over the same response-token prefixes. Only the supervised term is weighted by $w(h)$. We train LoRA parameters with the quantized base frozen; token-level definitions and fixed hyperparameters are in Appendix~\ref{app:architecture}.

\label{sec:theory}
\begin{proposition}[Outcome ordering and correction weight]
\label{prop:direction-strength}
Fix $r$, a symmetric prior $\alpha>0$, $\lambda\ge0$, and $v=(1,-1,0,0)$. Replace $E(p)$ by any mass $\tau>0$ in Equations~\ref{eq:posterior}--\ref{eq:utility}. The posterior mean satisfies
\begin{equation}
\mu_k(\tau)
=\gamma(\tau)r_k+\frac{1-\gamma(\tau)}{4},
\qquad
\gamma(\tau)=\frac{\tau}{4\alpha+\tau}.
\label{eq:shrinkage}
\end{equation}
Consequently, category ordering and the argmax set are invariant to $\tau$. Moreover, the correction weight $w_\tau$ is nondecreasing in $\tau$, so
\begin{equation}
w_{E(p)}\le w_n.
\label{eq:weight-attenuation}
\end{equation}
\end{proposition}
At fixed support, effective mass attenuates the supervised coefficient relative to nominal mass while preserving category ordering. Its effect on log loss depends on the observed category. Appendix~\ref{app:proofs} gives the proofs and the outcome-dependent log-loss condition.

\FloatBarrier
\section{Experiments}
\label{sec:experiments}

We evaluate three questions. \textbf{RQ1} examines whether additional visible executions improve outcome prediction. \textbf{RQ2} isolates how effective mass changes predictive probabilities at fixed relative support. \textbf{RQ3} tests the complete correction-learning update through fresh program generation. Figures~\ref{fig:history-scaling} and~\ref{fig:performance-summary} organize these measurements, with full numerical reports in the appendix.

\subsection{Experimental setup}

\paragraph{Tasks and models.}
We use RunBugRun and CodeARC, together with APPS Replay and CodeContests Replay \citep{runbugrun,codearc}. RunBugRun contains executable buggy programs; CodeARC tests program synthesis from behavioral observations. The replay domains are controlled synthesis extensions evaluated on the recorded source populations and execution budgets. CodeARC uses the fixed-source execution protocol specified in Appendix~\ref{app:protocol}; the original benchmark uses interactive differential testing \citep{codearc}. The downstream matrix covers Qwen2.5-Coder-7B-Instruct, DeepSeek-Coder-6.7B-Instruct, and Gemma-3-4B-it, with 500 primary sources per model--domain cell.

\paragraph{Measurements and comparisons.}
The prediction studies measure future-outcome negative log-likelihood (NLL) and categorical accuracy. The history sweep compares fixed mass $n$ with effective mass $E(p)$ as visible history grows. A separate matched comparison holds relative support and predictor parameters fixed, changing only the pseudo-count mass. Downstream evaluation generates one fresh candidate per source after one update; all-tests Pass@1 requires passing all ten evaluator tests. It compares full \method (\texttt{eesd\_full}) with the unchanged entering policy (\texttt{no\_update}). The paired effect, in percentage points, is
\begin{equation}
\Delta_{\mathrm{pp}}=100\bigl[S(\texttt{eesd\_full})-S(\texttt{no\_update})\bigr],
\label{eq:downstream-delta}
\end{equation}
where $S$ is the fraction of solved sources. Paired 95\% intervals use 10,000 whole-source bootstrap draws and quantify source-level uncertainty conditional on the evaluated policies. Evidence scoring uses public executions; hidden outcomes are reserved for evaluation. Fresh generation produces a new completion for each recorded source. Protocol and replication details are in Appendix~\ref{app:protocol}.

\paragraph{Training and hardware.}
We adapt LoRA parameters on a frozen 4-bit NF4 base with bfloat16 computation. The correction scorer uses $\alpha=0.1$ and $\lambda=0.5$; the forward-KL coefficient is $\beta=0.03$. Experiments use NVIDIA RTX A6000 hardware with 48\,GB nominal memory per GPU. Evidence weights are computed from public records before student fitting, and the frozen entering policy supplies the KL reference. Appendix~\ref{app:architecture} specifies the relevance kernel, response-token objective, and optimization settings.

\subsection{RQ1: Does additional visible evidence improve prediction?}
\label{sec:rq1}
\label{sec:history-scaling}

We first ask whether the available executions carry useful predictive information. Figure~\ref{fig:history-scaling} varies visible history size over $n\in\{1,2,4,8\}$ for DeepSeek and Qwen on RunBugRun and CodeARC. Here $n$ is the number of visible observations. The sweep measures the predictive value of a longer execution history.

\paragraph{Prediction improves with longer histories.}
EED NLL decreases at every measured step in all four model--domain cells. On RunBugRun, the endpoint changes are $0.4696\rightarrow0.2114$ for DeepSeek and $0.4737\rightarrow0.2113$ for Qwen. On CodeARC, they are $0.7103\rightarrow0.2892$ and $0.7785\rightarrow0.3415$. Table~\ref{tab:history-summary} expresses these changes as relative NLL reductions of $55.0$--$59.3\%$. Categorical accuracy improves as well: by $2.5$ and $1.9$ percentage points on RunBugRun, and by $5.5$ and $6.1$ points on CodeARC, respectively. Additional history therefore benefits both the probability assigned to the observed outcome and the categorical prediction. RQ3 evaluates the resulting policy update.

\paragraph{Effective mass adds a distinct comparison.}
Fixed-mass prediction also benefits from additional observations. At $n=8$, EED further reduces NLL relative to Fixed in all four cells, as summarized in Figure~\ref{fig:performance-summary}(a): $0.0046$ and $0.0075$ on RunBugRun, and $0.0113$ and $0.0024$ on CodeARC, for DeepSeek and Qwen respectively. All four values are absolute NLL differences. At $n=1$, the two constructions coincide because $p_1=1$ and $E(p)=n=1$; differences at shorter histories are generally small. Figure~\ref{fig:history-scaling} and Appendix~\ref{app:history-size} retain the complete trajectories. The history sweep varies the available observations; RQ2 holds relative support fixed to isolate the mass intervention.

\begin{table}[H]
\centering
\small
\caption{\textbf{Endpoint summary of the EED history sweep.} NLL reduction and accuracy change compare $n=1$ with $n=8$; $\Delta\mathrm{NLL}_{8}=\mathrm{NLL}_{\mathrm{Fixed}}-\mathrm{NLL}_{\mathrm{EED}}$ at $n=8$. The reductions are calculated from the reported rounded measurements.}
\setlength{\tabcolsep}{4pt}
\begin{tabular}{llcccc}
\toprule
Benchmark & Backbone & \shortstack{EED NLL\\$1\!\rightarrow\!8$} & \shortstack{Reduction\\(\%)} & \shortstack{Acc. change\\(pp)} & $\Delta\mathrm{NLL}_{8}$ $\uparrow$ \\
\midrule
RunBugRun & DeepSeek 6.7B & $0.4696\!\rightarrow\!0.2114$ & 55.0 & +2.5 & +0.0046 \\
RunBugRun & Qwen2.5 7B & $0.4737\!\rightarrow\!0.2113$ & 55.4 & +1.9 & +0.0075 \\
CodeARC & DeepSeek 6.7B & $0.7103\!\rightarrow\!0.2892$ & 59.3 & +5.5 & +0.0113 \\
CodeARC & Qwen2.5 7B & $0.7785\!\rightarrow\!0.3415$ & 56.1 & +6.1 & +0.0024 \\
\bottomrule
\end{tabular}
\label{tab:history-summary}
\end{table}

\begin{figure}[H]
\centering
\includegraphics[width=1\linewidth]{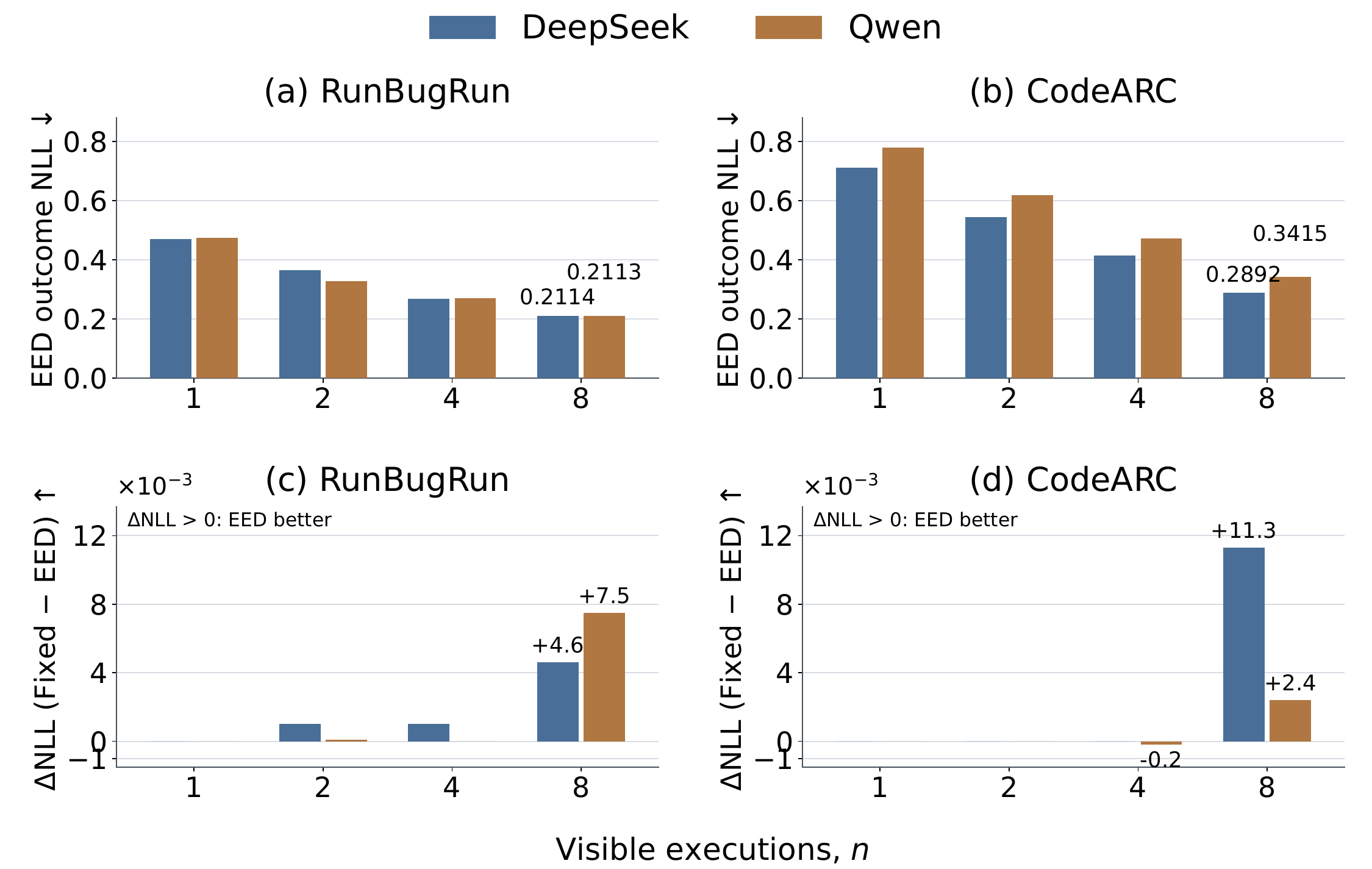}
\caption{\textbf{History size and effective evidence answer different questions.} (a,b) EED future-outcome NLL: lower is better. (c,d) $\Delta\mathrm{NLL}=\mathrm{NLL}_{\mathrm{Fixed}}-\mathrm{NLL}_{\mathrm{EED}}$: higher is better, in units of $10^{-3}$. All panels use $n=1,2,4,8$ visible observations. Differences are calculated from the four-decimal measurements in Table~\ref{tab:history-full}. The matched-parameter comparison is reported separately in Section~\ref{sec:rq2}.}
\label{fig:history-scaling}
\end{figure}

\subsection{RQ2: Does effective evidence change strength without changing direction?}
\label{sec:rq2}

The primary matched comparison uses DeepSeek/RunBugRun with the same relative support and effective-arm hyperparameters for both predictors. Only total pseudo-count mass changes. Their categorical argmax predictions agree on all 3,000 future-outcome examples, while effective mass reduces aggregate NLL by $0.00609$ with a 95\% source-bootstrap interval of $[0.00211,0.01075]$. Thus better probability estimates do not require changing the selected category. The ordering argument in Proposition~\ref{prop:direction-strength} applies to any categorical predictor with a symmetric prior, as detailed in Appendix~\ref{app:logloss}; the four-transition utility is the separate training construction.

\paragraph{The largest gain occurs under concentrated relevance.}
Table~\ref{tab:concentration} and Figure~\ref{fig:performance-summary}(b) divide the matched comparison into three effective-mass bins. The NLL reduction is $0.0159$ for $1\le E(p)<2$, $0.0077$ for $2\le E(p)<3$, and $0.0018$ for $3\le E(p)\le4$. The corresponding mean masses retain $37.8\%$, $63.0\%$, and $92.0\%$ of the nominal count of four (Appendix~\ref{app:derived-summaries}). In this comparison, the largest predictive benefit is associated with the largest reduction in pseudo-count mass. When support is already broadly distributed, effective and nominal mass are closer and the observed difference is smaller.

Changing mass adjusts shrinkage toward the symmetric prior at fixed observations and relative support. The favored category remains fixed while its assigned probability changes. The reported mass fractions quantify this pseudo-count intervention.

\begin{table}[htbp]
\centering
\small
\caption{\textbf{Concentration-stratified matched-mass comparison.} Primary DeepSeek/RunBugRun comparison. Gain is fixed-mass NLL minus EED NLL. Both predictors agree in argmax on all 3,000 examples.}
\setlength{\tabcolsep}{7pt}
\begin{tabular}{lccc}
\toprule
Effective-mass bin & Examples & Mean $E(p)$ & NLL gain $\uparrow$ \\
\midrule
$[1,2)$ & 568 & 1.511 & \textbf{+0.0159} \\
$[2,3)$ & 815 & 2.521 & +0.0077 \\
$[3,4]$ & 1617 & 3.681 & +0.0018 \\
\bottomrule
\end{tabular}
\label{tab:concentration}
\end{table}

\paragraph{From probability estimates to learning.}
The matched study isolates the evidence estimator. For training, the posterior is further mapped to a correction-specific utility and supervised coefficient. Proposition~\ref{prop:direction-strength} characterizes this mapping; RQ3 examines the complete resulting update. The other recorded mechanism repetitions and their intervals are provided in Appendix~\ref{app:mechanism-full}, separately from the history sweep. The log-loss analysis in Appendix~\ref{app:logloss} explains why a change in mass need not have the same benefit for every outcome.

\FloatBarrier
\subsection{RQ3: Does correction learning improve fresh generation?}
\label{sec:rq3}

We now move from outcome prediction to the coding agent's code-generating policy. Each evaluated program is freshly sampled after learning. Figure~\ref{fig:performance-summary}(c) presents four selected model--domain contrasts against the unchanged policy; Appendix~\ref{app:full-matrix} reports the complete twelve-cell evaluation.

\begin{figure}[htbp]
\centering
\includegraphics[width=\linewidth]{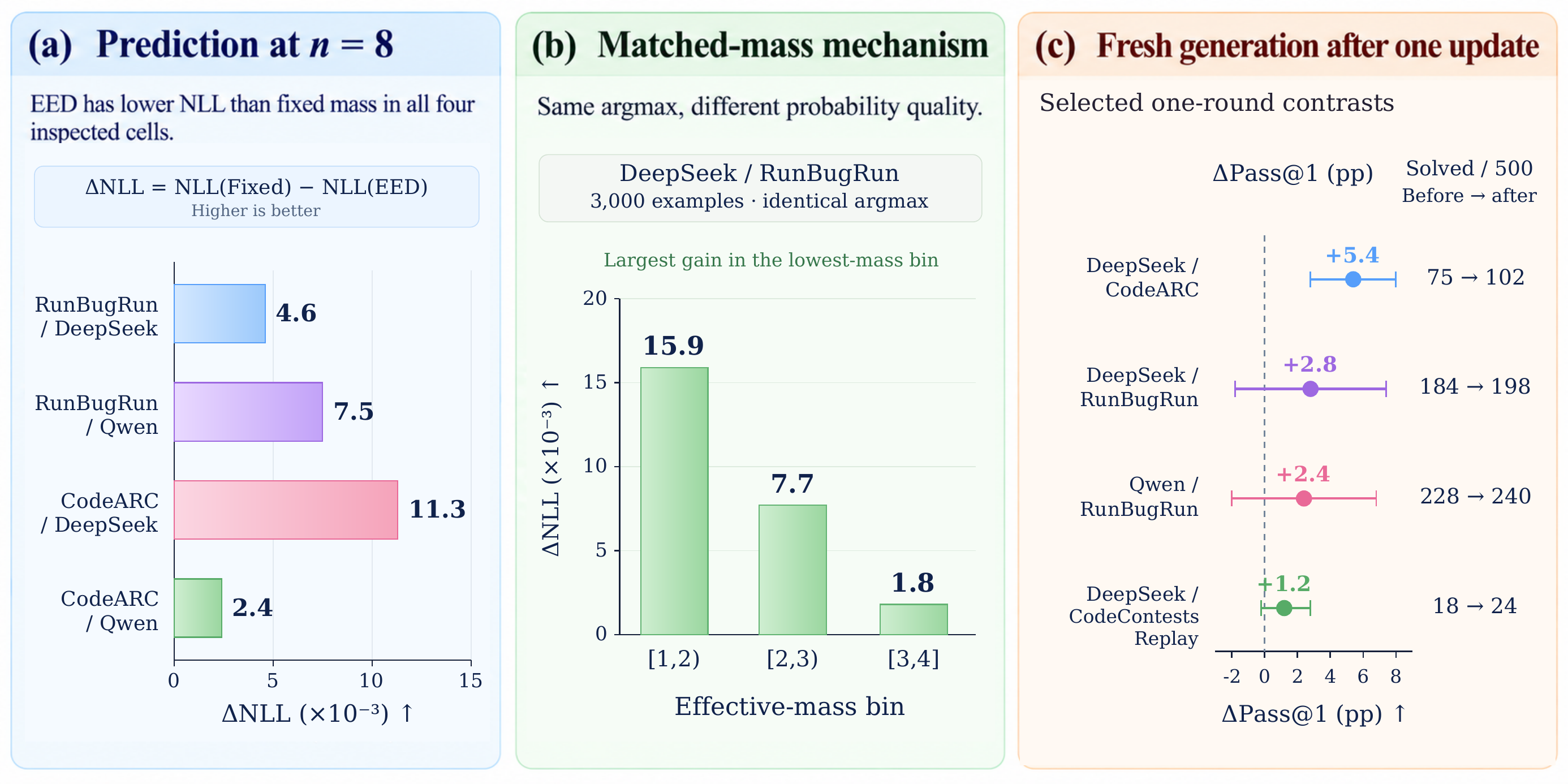}
\caption{\textbf{Prediction and fresh-generation performance.}
(a) Fixed-minus-EED NLL at $n=8$ in the four history-sweep cells.
(b) The matched DeepSeek/RunBugRun comparison by effective-mass bin; argmax agrees on all 3,000 examples. Both NLL panels use units of $10^{-3}$ and positive values favor EED.
(c) Four selected one-round contrasts against \texttt{no\_update}, with paired 95\% source-bootstrap intervals and solved counts out of 500. The complete model--domain evaluation is in Table~\ref{tab:full-downstream}.}
\label{fig:performance-summary}
\end{figure}

\paragraph{The clearest gain is on DeepSeek/CodeARC.}
All-tests Pass@1 rises from $75/500$ ($15.0\%$) to $102/500$ ($20.4\%$), a net increase of 27 solved sources. The paired improvement is $+5.4$ percentage points, with a 95\% source-bootstrap interval of $[+2.8,+8.0]$. Each solved program passes all ten evaluator tests. The comparison measures the effect of the complete correction-learning update.

\paragraph{Further gains appear on RunBugRun and CodeContests Replay.}
DeepSeek increases from $184/500$ to $198/500$ on RunBugRun ($+2.8$ points), and Qwen increases from $228/500$ to $240/500$ ($+2.4$ points). DeepSeek also increases from $18/500$ to $24/500$ on CodeContests Replay ($+1.2$ points). These correspond to 14, 12, and 6 additional solved sources in the respective aggregate counts. The paired intervals for these three contrasts span zero. Across all twelve settings, changes range from $-6.8$ to $+5.4$ points (Table~\ref{tab:full-downstream}), placing the DeepSeek/CodeARC gain within a model--domain-dependent response.

Together, the studies examine successive parts of the evidence-to-learning path: additional observations improve prediction, matched mass changes probability estimates without changing category ordering, and the complete weighted update yields positive fresh-generation point estimates in several model--domain settings. The downstream comparison evaluates the complete update; the matched-mass comparison isolates the probability estimator. 

\FloatBarrier
\section{Limitations}

Our evaluation covers executable coding tasks and one-round adaptation with a frozen lexical relevance kernel. Effective mass measures normalized relevance concentration; accounting for duplicated or dependent executions requires an additional dependence model. The uncertainty-penalized utility is a training score with no asserted confidence-coverage level. Downstream results compare the full update with the entering policy, so isolating the contributions of weighting and KL anchoring requires matched training controls. The reported source-bootstrap intervals condition on one fitting run per setting. Source-disjoint transfer and repeated RSI updates are further evaluation settings beyond the present protocol.

\section{Conclusion}
\label{sec:conclusion}

\method represents a correction's execution evidence through relative transition support and effective pseudo-count mass. A Dirichlet posterior maps this representation to an uncertainty-penalized supervised weight within KL-anchored self-distillation. The analysis establishes category-order invariance and the ordering of effective- and fixed-mass coefficients. Experiments show lower outcome NLL with longer histories, lower NLL than fixed mass in all four eight-observation comparisons, and a primary matched-mass gain at identical argmax predictions. One correction-learning round increases DeepSeek/CodeARC all-tests Pass@1 by $5.4$ percentage points. The full model--domain matrix defines the scope of this result. Together, these findings support explicit evidence-mass modeling when learning from a coding agent's self-corrections.

\label{pg:main-end}

\clearpage

\appendix

\section{Proofs and Interpretation of Effective Evidence}
\label{app:proofs}

\subsection{Pseudo-count construction and elementary properties}

\paragraph{Weighted-likelihood representation.}
Condition on the observed relevance weights $p$. Starting from a symmetric Dirichlet prior, the weighted likelihood
\begin{equation}
\widetilde L(\theta;h)
=\prod_{i=1}^{n}\theta_{z_i}^{E(p)p_i}
=\prod_{k=1}^{4}\theta_k^{E(p)r_k}
\end{equation}
gives
\begin{equation}
\widetilde p(\theta\mid h)
\propto
\prod_{k=1}^{4}\theta_k^{\alpha-1+E(p)r_k}.
\end{equation}
This is the density of $\mathrm{Dir}(\alpha\mathbf 1+E(p)r)$. This weighted likelihood provides an algebraic interpretation of the relevance-conditioned pseudo-count update.

\paragraph{Bounds and count recovery.}
Since $\sum_i p_i=1$ and $p_i\ge0$, Cauchy--Schwarz yields
\[
1=\left(\sum_i p_i\right)^2\le n\sum_i p_i^2,
\]
while $\sum_i p_i^2\le1$. Therefore $1\le E(p)\le n$. Equality at the upper bound requires uniform weights, for which $E(p)p_i=1$ and
\[
E(p)r_k=\sum_i\ind[z_i=k].
\]
Equality at the lower bound requires a point mass. Multiplying every $s_i$ by the same positive constant leaves $p$, $r$, $E(p)$, and the pseudo-count posterior unchanged.

\subsection{Posterior utility moments}

For a Dirichlet vector with mean $\mu$ and total concentration $\alpha_0$,
\[
\mathrm{Cov}(\theta)
=\frac{\mathrm{diag}(\mu)-\mu\mu^\top}{\alpha_0+1}.
\]
Thus $\E[v^\top\theta]=v^\top\mu$ and
\[
\mathrm{Var}(v^\top\theta)
=\frac{\sum_k\mu_kv_k^2-(v^\top\mu)^2}{\alpha_0+1}.
\]
Taking $v=(1,-1,0,0)$ gives Equations~\ref{eq:utility-mean} and~\ref{eq:utility-variance}. The training score $\bar u-\lambda\sigma_u$ combines these utility moments.

\subsection{Proof of Proposition~\ref{prop:direction-strength}}

\begin{proof}
Set $A=4\alpha$. The posterior mean is
\[
\mu_k(\tau)=\frac{\alpha+\tau r_k}{A+\tau}
=\frac{\tau}{A+\tau}r_k+\frac{A}{A+\tau}\frac14.
\]
For any categories $j,k$,
\[
\mu_j(\tau)-\mu_k(\tau)
=\frac{\tau(r_j-r_k)}{A+\tau}.
\]
Its sign is independent of $\tau>0$, proving invariance of category ordering, including ties.

To establish monotonicity of the training weight, write
\[
d=r_{\fix}-r_{\regress},
\qquad
b=r_{\fix}+r_{\regress}.
\]
Then
\begin{align}
\bar u_\tau&=\frac{\tau d}{A+\tau},\\
\sigma_\tau^2
&=\frac{(2\alpha+\tau b)(A+\tau)-\tau^2d^2}
{(A+\tau)^2(A+\tau+1)}.
\end{align}
If $d\le0$, then $\bar u_\tau\le0$ and $w_\tau=0$ for every $\tau$.

Suppose $d>0$. All Dirichlet parameters are positive, so $\sigma_\tau>0$. The ratio $T_\tau=\bar u_\tau/\sigma_\tau$ satisfies
\begin{equation}
T_\tau^2
=\frac{d^2(A+\tau+1)}
{b-d^2+(2\alpha+Ab)/\tau+2\alpha A/\tau^2}.
\label{eq:utility-snr}
\end{equation}
Because $|d|\le b\le1$, we have $b-d^2\ge0$. The numerator is strictly increasing in $\tau$, and the positive denominator is strictly decreasing. Therefore $T_\tau$ is strictly increasing. Since $\bar u_\tau$ is also positive and increasing,
\[
w_\tau
=\bar u_\tau\left[1-\frac{\lambda}{T_\tau}\right]_+
\]
is nondecreasing for every $\lambda\ge0$. Finally, $E(p)\le n$ gives $w_{E(p)}\le w_n$.
\end{proof}

The proposition characterizes the supervised coefficient under fixed $r$, $\alpha$, $\lambda$, and utility vector.

\subsection{When reducing evidence mass changes log loss}
\label{app:logloss}

For a predictive label set of size $K$, the symmetric-prior construction is
$\mu_k(\tau)=(\alpha+\tau r_k)/(K\alpha+\tau)$.
The difference $\mu_j-\mu_k=\tau(r_j-r_k)/(K\alpha+\tau)$ preserves
category ordering for every $K$. The category-order statement holds for any predictive label set.
The correction-utility result uses $K=4$ and $v=(1,-1,0,0)$.

For an observed transition $y$, define
\[
\ell_\tau(y)=-\log\mu_y(\tau)
=-\log\frac{\alpha+\tau r_y}{4\alpha+\tau}.
\]
Differentiation gives
\begin{equation}
\frac{\partial\ell_\tau(y)}{\partial\tau}
=\frac{\alpha(1-4r_y)}
{(\alpha+\tau r_y)(4\alpha+\tau)}.
\end{equation}
Consequently, when $E(p)<n$, reducing mass from $n$ to $E(p)$ lowers log loss exactly when $r_y<1/4$, increases it when $r_y>1/4$, and leaves it unchanged when $r_y=1/4$. The log-loss effect is therefore outcome-dependent: empirical gains reflect how relevance concentration aligns with errors in relative support.

\subsection{Interpretive limits of effective mass}

Consider replacing every observation by $m$ identical copies. Each copy then has normalized weight $p_i/m$, so
\[
E(p')=\left(\sum_i\sum_{\ell=1}^{m}(p_i/m)^2\right)^{-1}
=mE(p),
\qquad r'=r.
\]
Duplicating observations increases mass while leaving relative support fixed. Effective mass consequently summarizes normalized relevance concentration, while dependence and absolute score reliability are separate properties of the execution record.

\section{Execution Protocol and Implementation}
\label{app:architecture}
\label{app:protocol}

\subsection{Measurement protocols}

\begin{table}[H]
\centering
\small
\caption{\textbf{Evaluation protocols and outcomes.} Candidate selection, outcome prediction, and post-update generation use the protocols below.}
\setlength{\tabcolsep}{5pt}
\renewcommand{\arraystretch}{1.15}
\begin{tabular}{>{\raggedright\arraybackslash}p{0.20\linewidth}>{\raggedright\arraybackslash}p{0.40\linewidth}>{\raggedright\arraybackslash}p{0.30\linewidth}}
\toprule
Study & Recorded protocol & Reported outcome \\
\midrule
Candidate-bank audit & 500 source groups per benchmark; eight candidates and four public executions per group & Visible-selected and retrospective hidden-oracle success rates \\
History-size prediction & $n\in\{1,2,4,8\}$ visible observations; RunBugRun and CodeARC; DeepSeek and Qwen; primary run & Future-outcome NLL and categorical accuracy \\
Matched-mass prediction & DeepSeek/RunBugRun; same relative support and effective-arm settings within each contrast & NLL difference and argmax agreement; three repetitions \\
One-round learning & Three backbones and four domains; 500 primary sources per cell; one fitting run and one fresh candidate per source & Pass@1 requiring all ten evaluator tests; paired source-bootstrap interval \\
\bottomrule
\end{tabular}
\label{tab:protocols}
\end{table}

The retrospective oracle in Appendix~\ref{app:candidate-bank} selects within a fixed candidate bank using evaluator outcomes. The downstream reference is the unchanged entering policy. Outcome prediction is scored with NLL and categorical accuracy; post-update generation is scored with all-tests Pass@1. Evaluation uses the recorded source populations. The supplied aggregate reports leave the overlap between training and evaluator source populations unspecified.

\subsection{Randomness and uncertainty}
\label{app:randomness}

The recorded random-number-generator seeds are 1701 for the primary
history sweep, matched-mass comparison, and downstream evaluation, and
1702 and 1703 for the additional matched-mass repetitions. Each downstream
model--domain cell reports one fitting run. Its interval is a nominal
per-cell 95\% paired source-bootstrap interval conditional on the evaluated
policies. We use the recorded 10,000-draw intervals and report fitting-run
replication separately from source-level uncertainty.

\subsection{Relevance scoring and token-level objective}

Let $\Delta a$ be the unified diff between $a^0$ and $a^1$, and let $f$ be the frozen 256-dimensional signed-hash text encoder. Execution descriptor $e_i$ contains the input, expected output, actual output, stderr, and categorical outcome. The one-round correction scorer uses
\begin{equation}
s_i=\exp\!\left(\kappa\cos\bigl(f(\Delta a),f(e_i)\bigr)\right),
\qquad \kappa=16.
\label{eq:relevance-kernel}
\end{equation}
The scorer is frozen and uses public execution descriptors. Equation~\ref{eq:relevance-kernel} specifies the one-round correction scorer; the prediction studies retain their recorded matched-comparison settings.

For scored correction-response length $L_h$, let $c_{h,j}$ be the actual training prompt followed by $a^1_{<j}$. The two token-mean losses are
\begin{align}
\mathcal L_{\mathrm{CE}}(h)
&=-\frac{1}{L_h}\sum_{j=1}^{L_h}
\log\pi_{t+1}(a^1_j\mid c_{h,j}),\\
D_{\mathrm{KL}}^{(h)}(\pi_t\Vert\pi_{t+1})
&=\frac{1}{L_h}\sum_{j=1}^{L_h}
D_{\mathrm{KL}}\!\left(
\pi_t(\cdot\mid c_{h,j})\,\Vert\,
\pi_{t+1}(\cdot\mid c_{h,j})
\right).
\end{align}
Teacher and student use the same response prefixes. The execution record determines the scalar $w(h)$, and the training prompt supplies the token context. At $w(h)=0$, the objective retains its KL term.

\subsection{Optimization settings}

\begin{table}[H]
\centering
\small
\caption{\textbf{Recorded one-round training configuration.} GPU memory denotes nominal capacity.}
\setlength{\tabcolsep}{6pt}
\begin{tabular}{ll}
\toprule
Setting & Value \\
\midrule
Hardware & NVIDIA RTX A6000; 48\,GB per GPU \\
Student adaptation / base & LoRA only / frozen \\
Quantization & 4-bit NF4, double quantization \\
Compute dtype & bfloat16 \\
LoRA rank / scaling alpha / dropout & 16 / 32 / 0 \\
Optimizer / learning rate & AdamW / $2\times10^{-4}$ \\
Gradient accumulation / clipping & 16 / 1.0 \\
Maximum sequence length & 4608 \\
Relevance encoder / $\kappa$ & 256-dimensional signed hash / 16 \\
Dirichlet prior $\alpha$ & 0.1 \\
Utility / uncertainty coefficient & $(1,-1,0,0)$ / $\lambda=0.5$ \\
Forward-KL coefficient $\beta$ & 0.03 \\
Loss normalization & Response tokens \\
Training budget policy & Response-token accounting \\
\bottomrule
\end{tabular}
\label{tab:training-settings}
\end{table}

\subsection{One-correction algorithm}
\label{app:algorithm}

Given $(a^0,a^1)$ and their public before/after execution outcomes:
\begin{enumerate}
\item Compute $s_i$ using Equation~\ref{eq:relevance-kernel} and normalize $p_i=s_i/\sum_j s_j$.
\item Assign each outcome pair its transition $z_i$ and aggregate $r_k=\sum_i p_i\ind[z_i=k]$.
\item Set $E(p)=1/\sum_i p_i^2$ and $\alpha_k^{\mathrm{post}}=\alpha+E(p)r_k$.
\item Calculate $\bar u$ and $\sigma_u$ under $v=(1,-1,0,0)$, then set $w=[\bar u-\lambda\sigma_u]_+$.
\item Fit the LoRA student with $w\mathcal L_{\mathrm{CE}}+\beta D_{\mathrm{KL}}^{(h)}(\pi_t\Vert\pi_{t+1})$ while the entering-policy teacher is frozen.
\end{enumerate}

\section{Candidate-Bank Selection Diagnostic}
\label{app:candidate-bank}

For each of 500 source groups per benchmark, visible-score selection chooses among eight fixed candidates using four public executions. A \emph{retrospective hidden oracle} selects the candidate with the highest additional evaluator score in that bank. This evaluation-only diagnostic measures selection headroom among the generated candidates.

\begin{table}[H]
\centering
\small
\caption{\textbf{Selection headroom within a fixed candidate bank.} Both selectors operate on the same eight generated candidates; oracle selection uses evaluator outcomes.}
\setlength{\tabcolsep}{8pt}
\begin{tabular}{lccc}
\toprule
Benchmark & Visible selected (\%) & Hidden oracle (\%) & Gap (pp) \\
\midrule
CodeARC & 27.0 & 31.2 & +4.2 \\
RunBugRun & 71.4 & 75.4 & +4.0 \\
\bottomrule
\end{tabular}
\label{tab:bank-headroom}
\end{table}

The gaps quantify the selection headroom available within each fixed candidate bank.

\section{Complete One-Round Downstream Results}
\label{app:full-matrix}

\begin{table}[H]
\centering
\small
\caption{Complete one-round downstream matrix for the primary run. Each model--domain setting contains 500 sources and reports fresh all-tests Pass@1. Intervals are paired source-bootstrap 95\% intervals.}
\setlength{\tabcolsep}{4pt}
\begin{tabular}{llcccc}
\toprule
Benchmark & Backbone & \texttt{no\_update} $\uparrow$ & \method $\uparrow$ & \shortstack{$\Delta$ (pp)\\$\uparrow$} & \shortstack{95\% CI\\(pp)} \\
\midrule
\multirow{3}{*}{RunBugRun} & Qwen2.5 Coder 7B & 228/500 & 240/500 & +2.4 & [-2.0,+6.8] \\
 & DeepSeek Coder 6.7B & 184/500 & 198/500 & +2.8 & [-1.8,+7.4] \\
 & Gemma 3 4B & 142/500 & 132/500 & -2.0 & [-5.6,+1.4] \\
\midrule
\multirow{3}{*}{CodeARC} & Qwen2.5 Coder 7B & 85/500 & 83/500 & -0.4 & [-2.8,+2.0] \\
 & DeepSeek Coder 6.7B & 75/500 & \textbf{102/500} & \textbf{+5.4} & \textbf{[+2.8,+8.0]} \\
 & Gemma 3 4B & 59/500 & 56/500 & -0.6 & [-2.6,+1.4] \\
\midrule
\multirow{3}{*}{APPS Replay} & Qwen2.5 Coder 7B & 57/500 & 57/500 & +0.0 & [-1.8,+1.8] \\
 & DeepSeek Coder 6.7B & 23/500 & 17/500 & -1.2 & [-3.2,+0.8] \\
 & Gemma 3 4B & 36/500 & 2/500 & -6.8 & [-9.0,-4.8] \\
\midrule
\multirow{3}{*}{CodeContests Replay} & Qwen2.5 Coder 7B & 47/500 & 39/500 & -1.6 & [-3.2,+0.0] \\
 & DeepSeek Coder 6.7B & 18/500 & 24/500 & +1.2 & [-0.2,+2.8] \\
 & Gemma 3 4B & 31/500 & 4/500 & -5.4 & [-7.6,-3.2] \\
\bottomrule
\end{tabular}
\label{tab:full-downstream}
\end{table}

Table~\ref{tab:full-downstream} reports all twelve model--domain settings. DeepSeek has positive point estimates in three domains; Qwen improves on RunBugRun and matches its entering-policy solved count on APPS Replay. The full range of paired differences is $[-6.8,+5.4]$ percentage points. Intervals quantify source-level uncertainty for the primary fitting run.

\section{History-Size Mechanism Details}
\label{app:history-size}

\begin{table}[H]
\centering
\small
\caption{Primary history-size sweep with both fixed and effective mass. Values are future-outcome NLL; lower is better. Displayed precision is four decimals.}
\setlength{\tabcolsep}{3.5pt}
\begin{tabular}{llcccccccc}
\toprule
& & \multicolumn{2}{c}{$n=1$} & \multicolumn{2}{c}{$n=2$} & \multicolumn{2}{c}{$n=4$} & \multicolumn{2}{c}{$n=8$} \\
\cmidrule(lr){3-4}\cmidrule(lr){5-6}\cmidrule(lr){7-8}\cmidrule(lr){9-10}
Benchmark & Backbone & Fixed & EED & Fixed & EED & Fixed & EED & Fixed & EED \\
\midrule
RunBugRun & DeepSeek & .4696 & .4696 & .3656 & .3646 & .2703 & .2693 & .2160 & \textbf{.2114} \\
RunBugRun & Qwen & .4737 & .4737 & .3292 & .3291 & .2704 & .2704 & .2188 & \textbf{.2113} \\
CodeARC & DeepSeek & .7103 & .7103 & .5440 & .5440 & .4149 & .4149 & .3005 & \textbf{.2892} \\
CodeARC & Qwen & .7785 & .7785 & .6176 & .6176 & .4725 & .4727 & .3439 & \textbf{.3415} \\
\bottomrule
\end{tabular}
\label{tab:history-full}
\end{table}

Table~\ref{tab:history-full} supplies Figure~\ref{fig:history-scaling}. The history-size sweep and the matched-mass study in Appendix~\ref{app:mechanism-full} retain their respective recorded predictor settings. The DeepSeek/RunBugRun difference at $n=4$ is $0.0010$ in the sweep; the primary matched-mass contrast is $0.00609$. Displayed ties refer to the four-decimal reporting precision.

\begin{table}[H]
\centering
\small
\caption{\textbf{Categorical accuracy (\%) in the reported history-size sweep.} Values are reported to one decimal place. Accuracy is measured on future categorical outcomes.}
\setlength{\tabcolsep}{7pt}
\begin{tabular}{llrrrr}
\toprule
Benchmark & Backbone & $n=1$ & $n=2$ & $n=4$ & $n=8$ \\
\midrule
RunBugRun & DeepSeek 6.7B & 90.7 & 90.9 & 92.2 & 93.2 \\
RunBugRun & Qwen2.5 7B & 90.6 & 91.4 & 91.5 & 92.5 \\
CodeARC & DeepSeek 6.7B & 83.3 & 83.9 & 84.7 & 88.8 \\
CodeARC & Qwen2.5 7B & 80.5 & 81.0 & 83.7 & 86.6 \\
\bottomrule
\end{tabular}
\label{tab:history-accuracy}
\end{table}

\section{Matched-Mass Mechanism and Replication}
\label{app:mechanism-full}

\paragraph{Concentration bins.}
Table~\ref{tab:concentration} and Figure~\ref{fig:performance-summary}(b) report the same primary DeepSeek/RunBugRun comparison at matched effective-arm settings. The bins contain 3,000 examples in total, with identical argmax predictions between estimators. Positive NLL differences favor effective mass. The table reports the recorded bin-level point estimates.

\paragraph{Recorded repetitions.}
Table~\ref{tab:cross-seed-mechanism} gives all three recorded repetitions. Each contrast matches the effective-arm parameters between predictors. The table records the initialization-specific contrasts under the outcome-dependent log-loss characterization in Appendix~\ref{app:logloss}. Initialization identifiers are defined in Appendix~\ref{app:randomness}.

\begin{table}[H]
\centering
\small
\caption{DeepSeek/RunBugRun mechanism sensitivity across recorded initializations at the effective-arm parameters. Positive NLL gain favors effective evidence over fixed mass.}
\setlength{\tabcolsep}{7pt}
\begin{tabular}{ccc}
\toprule
Initialization & NLL gain $\uparrow$ & 95\% CI \\
\midrule
1701 & \textbf{+0.00609} & [+0.00211,+0.01075] \\
1702 & -0.00367 & [-0.00447,-0.00287] \\
1703 & +0.00067 & [+0.00023,+0.00118] \\
\bottomrule
\end{tabular}
\label{tab:cross-seed-mechanism}
\end{table}

\section{Derived Arithmetic and Paired-Count Interpretation}
\label{app:derived-summaries}

\subsection{Coding illustration and supervision weights}
\label{app:coding-illustration}
In Figure~\ref{fig:problem}, \texttt{sum(range(1, n))} gives $0,1,3,6$
on inputs $1,2,3,4$, while \texttt{sum(range(1, n+1))} gives
$1,3,6,10$. All transitions are \fix, so $r=(1,0,0,0)$ for either
stipulated relevance vector: $(0.25,0.25,0.25,0.25)$ or
$(0.97,0.01,0.01,0.01)$. The relevance vectors define the analytical
example independently of the lexical kernel. Their masses are $4$ and
$1/0.9412\approx1.062473$. With $\alpha=0.1$, $\lambda=0.5$, and
$v=(1,-1,0,0)$, Equations~\ref{eq:posterior}--\ref{eq:utility} give
$w\approx0.832081$ and $0.541946$. Fixed mass assigns $0.832081$ to
both. Targets and KL coefficients are held fixed
while the supervised coefficients change with effective mass.

\subsection{Reported mass fractions and paired counts}

For the four-observation matched comparison in Table~\ref{tab:concentration},
the retained pseudo-count fraction is $E(p)/4$. The three bin means therefore
retain $100(1.511/4)=37.775\%$, $100(2.521/4)=63.025\%$, and
$100(3.681/4)=92.025\%$ of fixed mass. The main text rounds these to one
decimal place. These ratios summarize retained pseudo-count mass.

For paired downstream outcomes, let $N_{01}$ count sources that fail under
\texttt{no\_update} and pass under \method, and $N_{10}$ count the reverse.
The difference in solved-source counts equals $N_{01}-N_{10}$, giving
\[
\Delta_{\mathrm{pp}}=100\frac{N_{01}-N_{10}}{N},\qquad N=500.
\]
Aggregate solved counts identify the net difference $N_{01}-N_{10}$.
The paired uncertainty intervals are taken from the recorded whole-source
bootstrap comparisons.

\end{document}